\documentclass[11pt,letterpaper]{article}
\usepackage[top=1in, bottom=1in, left=1in, right=1in]{geometry}

\usepackage[utf8]{inputenc} 
\usepackage[T1]{fontenc}    
\usepackage{hyperref}       
\usepackage{url}            
\usepackage{booktabs}       
\usepackage{amsfonts}       
\usepackage{nicefrac}       
\usepackage{microtype}      
\usepackage{graphicx}
\usepackage{enumitem}
\usepackage{natbib}
\usepackage{algorithm}
\usepackage[noend]{algpseudocode}

\makeatletter
\def\algbackskip{\hskip-\ALG@thistlm}
\makeatother

\title{Exact Partitioning of High-order Planted Models \\ with a Tensor Nuclear Norm Constraint}

\author{
  \textbf{Chuyang Ke}\\Department of Computer Science\\Purdue University\\\texttt{cke@purdue.edu}
  \and 
  \textbf{Jean Honorio}\\Department of Computer Science\\Purdue University\\\texttt{jhonorio@purdue.edu}
}

\usepackage{amsmath, amsthm, amssymb}
\usepackage{bm}
\usepackage{color}
\usepackage{ulem}

\let\emph\textit

\def\R{{\mathbb{R}}}

\def\1{\textbf{1}}

\newcommand{\onevct}{\bm{1}}

\newcommand{\Imtx}{\mathcal{I}}

\newtheorem{theorem}{Theorem}

\newtheorem{lemma}{Lemma}
\newtheorem{definition}{Definition}

\newcommand{\Prob}[2][]{\mathbb{P}_{#1}\left\{ {#2} \right\}}
\newcommand{\Expect}[2][]{\mathbb{E}_{#1}\left[ #2 \right]}

\newcommand{\abs}[1]{\left\vert {#1} \right\vert}

\newcommand{\vertiii}[1]{{\left\vert\kern-0.25ex\left\vert\kern-0.25ex\left\vert #1 \right\vert\kern-0.25ex\right\vert\kern-0.25ex\right\vert}}
\newcommand{\tnorm}[1]{\vertiii{#1}}

\newcommand{\tnormone}[1]{{\tnorm{#1}}_\text{1}}
\newcommand{\tnormnu}[1]{{\tnorm{#1}}_{\ast}}
\newcommand{\tnorminf}[1]{{\tnorm{#1}}_{\infty}}

\newcommand{\spa}[1]{\operatorname*{span}\left({#1}\right)}

\newcommand{\proj}{\bm{P}}
\newcommand{\pq}{\mathcal{Q}}

\newcommand{\st}{\operatorname*{subject\; to}}
\newcommand{\maximize}{\operatorname*{maximize}}

\newcommand{\vct}[1]{{#1}}

\newcommand{\ten}[1]{\mathcal{#1}}

\newcommand{\vy}{\vct{y}}

\newcommand{\om}{{\otimes m}}

\newcommand{\inprod}[2][]{\left\langle {#1},{#2} \right\rangle}

\newcommand{\ta}{\ten{A}}
\newcommand{\tb}{\ten{B}}
\newcommand{\tc}{\ten{C}}

\newcommand{\tx}{\ten{X}}
\newcommand{\ty}{\ten{Y}}
\newcommand{\tz}{\ten{Z}}

\newcommand{\tw}{\ten{W}}
\newcommand{\dty}{\ty^\ast - \ty}
\newcommand{\tya}{{\ten{Y}^\ast}}
\def\L{{\mathcal{L}}}

\newcommand{\sphere}{{\mathbb{S}^{n-1}}}

\date{}
\begin{document}
\maketitle

\begin{abstract}
We study the problem of efficient exact partitioning of the hypergraphs generated by high-order planted models. A high-order planted model assumes some underlying cluster structures, and simulates high-order interactions by placing hyperedges among nodes. Example models include the disjoint hypercliques, the densest subhypergraphs, and the hypergraph stochastic block models. We show that exact partitioning of high-order planted models (a NP-hard problem in general) is achievable through solving a computationally efficient convex optimization problem with a tensor nuclear norm constraint. Our analysis provides the conditions for our approach to succeed on recovering the true underlying cluster structures, with high probability. 
\end{abstract}

\allowdisplaybreaks

\section{Introduction}
On a higher level, a \emph{planted model} simulates interactions among various groups of entities in a network. Typical planted models assume that nodes are grouped into a number of clusters, and each pair of nodes is connected randomly with some probability related to the cluster membership. The generative and non-deterministic nature makes planted models of both theoretical and practical interests in the field of community detection, data mining, engineering, biology, among others. Various classical planted models have been studied extensively in recent years. 
This includes, for instance, the stochastic block models (SBMs) \citep{mossel2016belief,abbe2015exact,abbe2015community,amini2018semidefinite}, the planted cliques \citep{barak2019nearly, ames2011nuclear}, the densest subgraphs \citep{fang2019efficient, arias2014community}, the latent space models \citep{chen2020near, ke2018information}.

However in more complex real-world systems of interest, entities may interact beyond the pairwise way. One example is the folksonomy, where a group of individuals collaboratively annotate a data set to create semantic structure \citep{ghoshal2009random}. Such networks usually exhibit a “user–resource–annotation” tripartite structure, and naturally can be modeled as a 3-uniform hypergraph \citep{ghoshdastidar2017consistency}. In the past decades researchers have been utilizing hypergraphs to model a number of real-life networks, for instance, brain regions \citep{gu2017functional,zu2016identifying}, food webs \citep{li2017inhomogeneous}, images \citep{li2013contextual}, VLSI designs \citep{karypis1999multilevel}. Hypergraphs are the generalization of ordinary (pairwise) graphs, where each hyperedge connects multiple nodes. By introducing hyperedges one may capture potential higher-order interactions among nodes, which are rather prevalent in the aforementioned real-life tasks.

The problem of partitioning high-order models have been studied for some long time. Graph-theoretic problems including cuts, colorings, and traversals have been analyzed in earlier works \citep{berge1984hypergraphs, karypis1999multilevel}. In the past decade, researchers have started looking at spectral theory and algebraic connectivity of hypergraphs \citep{pearson2014spectral,cooper2012spectra,hu2012algebraic, zhou2007learning}.
Certain high-order models of interests have received more attention from an algorithmic point of view. This includes the densest subhypergraphs \citep{buhmannrecovery, chlamtac2016densest, taylor2016approximations}, the hypergraph SBMs \citep{kim2017community, florescu2016spectral}, and the hypergraph planted cliques \citep{zhang2018tensor}.
Despite years of research, however, very little is known about the efficient \emph{exact partitioning} conditions in general high-order planted models.

In this paper we propose an efficient convex optimization approach for exact partitioning of high-order planted models. A high-order planted model generates hypergraphs, which simulate multi-entity interactions in a network. Our model class formulation is highly general and subsumes models analyzed and applied in prior literature, including the disjoint hypercliques, the densest subhypergraphs, and the hypergraph stochastic block models. When the order is set to $2$, our definition of high-order planted models reduces to regular planted models with ordinary graphs. 

It is known that with certain constraints such as balancedness, the problem of computing an optimal partitioning of a hypergraph is NP-hard in general \citep{BORNDORFER201515, lyaudet2010np}. In spite of being computationally hard, we provide an efficient exact partitioning algorithm, which recovers the underlying cluster structures with probability tending to $1$ if certain statistical conditions are fulfilled. Hypergraph partitioning algorithms in prior literature either are objective function approximation algorithms \citep{chlamtac2016densest,taylor2016approximations}, or unfold hypergraphs into matrices \citep{kim2017community, ghoshdastidar2017consistency}. On the contrary by ``exact partitioning'', our algorithm returns the true solution perfectly (up to the permutation of clusters) using a novel tensor optimization approach. Our exact partitioning algorithm is formulated as a convex optimization problem, which can be solved efficiently using interior point methods \citep{boyd2004convex}. 

Motivated by the use of adjacency matrices in ordinary graphs, researchers have been using tensors, or multidimensional arrays, to embed the information in hypergraphs. When dealing with tensors, one common approach is to unfold the tensor into matrices, and apply partitioning algorithms for ordinary graphs \citep{lu2019tensor, zhou2007learning}. There also exists some prior works using Sum-of-Squares (SoS) based relaxations \citep{kim2017community}. In this paper however, we let tensors be tensors; we are interested in generic tensor methods. Our proof relies on a careful construction of tensor projections and a novel analysis of tensor nuclear norm constraints. It is worth mentioning that tensor nuclear norms have been used extensively in problems related to tensor completion \citep{yuan2017incoherent, yuan2016tensor}. To the best of our knowledge, the use of tensor nuclear norm in hypergraph partitioning problems is novel. 

The feasibility of efficient exact partitioning depends on the signal-to-noise ratio (SNR) in the planted model. In our model the SNR is determined by two signal parameters $p$ and $q$ (see Definition \ref{def:model} for details). We show that in high-order planted models, the problem becomes statistically easier if the gap between $p$ and $q$ becomes larger. Intuitively, a larger SNR (i.e., larger gap between $p$ and $q$) leads to cleaner cluster structures in the observed hypergraph and the adjacency tensor. The generative nature enables us to study the average-case behaviors of high-order planted models. Our analysis establishes the regime in which efficient exact recovery of hidden cluster structures is possible from noisy observation of hypergraphs.

\textbf{Related Works.} There has been a lot of research on the partitioning of certain high-order planted models. 
For the densest subhypergraphs, \citet{chlamtac2016densest} and \citet{taylor2016approximations} proposed theoretical objective function approximation algorithms. Our goal, arguably more challenging, is to recover the true underlying clustering structure. For hypergraph SBMs, approaches include truncating the hypergraph to a multigraph \citep{kim2017community} or an ordinary graph with a weighted adjacency matrix \citep{ghoshdastidar2017consistency}.
It is worth highlighting that our result is not merely an extension. Our definition of high-order planted models is highly general, and the convex optimization formulation with a tensor nuclear norm constraint is novel. Moreover our approach does not approximate, or truncates the hypergraph to an ordinary graph. To the best of authors' knowledge, we are the first one applying tensor nuclear norm methods in high-order partitioning problems. 

\textbf{Summary of our contributions.} We provide a series of novel results in this paper:
\begin{itemize}
    \item We propose the highly general class definition of high-order planted models, and demonstrate that our model class definition subsumes several existing planted models, including the disjoint hypercliques, the densest subhypergraphs, and the hypergraph stochastic block models. 
    \item We formulate the problem of exact partitioning in high-order planted models as a novel tensor optimization problem with a tensor nuclear norm constraint, and we establish the regime in which hidden cluster structures can be recovered efficiently. 
\end{itemize}


\section{Preliminaries}
\subsection{Notations}
\label{section:notation}
In this section, we introduce the notations that will be used in the paper.

We use lowercase font (e.g., $a,b,u,v$) for scalars and vectors, uppercase font (e.g., $A, B, C$) for matrices, and calligraphic font (e.g., $\ta, \tb, \tc$) for tensors. We use $\R$ to denote the set of real numbers.
 
For any integer $n$, we use $[n]$ to denote the set $\{1, \ldots, n\}$.
For clarity when dealing with a sequence of objects, we use the superscript ${(i)}$ to denote the $i$-th object in the sequence, and subscript $j$ to denote the $j$-th entry. For example, for a sequence of vectors $\{x^{(i)}\}_{i \in [n]}$, $x^{(1)}_2$ represents the second entry of vector $x^{(1)}$. The notation $\otimes$ is used to denote outer product of vectors, for example, $x^{(1)} \otimes \ldots \otimes x^{(m)}$ is a tensor of order $m$, such that 
\begin{align*}
(x^{(1)} \otimes \ldots \otimes x^{(m)})_{i_1, \ldots, i_m} := x^{(1)}_{i_1} \ldots x^{(m)}_{i_m} \,.
\end{align*}
We use $\textbf{1}$ to denote the all-one vector.

Let $\ta = (\ta_{i_1, \dots, i_m})$ be an $m$-order tensor of size ${n_1 \times \cdots \times n_m}$. 
Tensor $\ta$ is \emph{symmetric} if it is invariant under any permutation of its indices, i.e., 
\[
\ta_{\sigma(i_1),\ldots, \sigma(i_m)} = \ta_{i_1,\ldots, i_m} \,,
\]
for any permutation $\sigma: [m] \to [m]$. 

For tensor $\ta = (\ta_{i_1, \dots, i_m})$ and $\tb = (\tb_{i_1, \dots, i_m})$ of same size, we define the \emph{inner product} of $\ta$ and $\tb$ as
\[
\inprod[\ta]{\tb} := \sum_{i_1,\ldots,i_m} \ta_{i_1 ,\ldots, i_m} \tb_{i_1 ,\ldots, i_m} \,.
\]

Tensor addition and subtraction are defined entrywise, e.g., $\ta + \tb = (\ta_{i_1, \dots, i_m} + \tb_{i_1, \dots, i_m})$. With slight abuse of notation, for any constant $c \in \R$, we use $\ta < c$ to denote entrywise inequality.


For any vector $u \in \R^n$, we denote the corresponding $m$-th order $\emph{rank-one}$ tensor as $u^{\om}$, where  
\[
(u^{\om})_{i_1 ,\ldots, i_m} := u_{i_1} \ldots u_{i_m} \,.
\]


For any symmetric tensor $\ta$, we define its \emph{spectral norm} as 
\[
\tnorm{\ta} := \sup_{u \in \sphere} \abs{\inprod[\ta]{u^\om}} \,,
\]
where $\sphere$ denotes the unit sphere. Similarly we define its \emph{nuclear norm} as 
\[
\tnormnu{\ta} := \inf \left\{ \sum_{i=1}^r \abs{\lambda_i} : A = \sum_{i=1}^r \lambda_i u^{(i) \om}, u^{(i)} \in \sphere\right\} \,.
\]
It is worth mentioning that, like the Schatten $p$-norms in the matrix case, the tensor spectral norm and tensor nuclear norm are also dual to each other \citep{friedland2018nuclear}.

Regarding entrywise norms, we define the tensor $L_1$ norm and the tensor $L_\infty$ norm of $\ta$ respectively as 
\[
\tnormone{\ta} := \sum_{i_1,\dots,i_m} \abs{\ta_{i_1,\dots,i_m}} \,, \qquad
\tnorminf{\ta} := \max_{i_1,\dots,i_m} \abs{\ta_{i_1,\dots,i_m}} \,.
\]

A tensor \emph{fiber} is a vector obtained by holding all but one mode constant. For example, using MATLAB-style notations, vector $\ta(i_1, \dots, i_{j-1}, :, i_{j+1}, \dots, i_m)$ is a mode-$j$ fiber of $\ta$. 
Define 
\[
\L_j (\ta) := \spa{\{\ta(i_1, \dots, i_{j-1}, :, i_{j+1}, \dots, i_m)\}_{i_k \in [n_k]}}
\]
as the vector space spanned by all mode-$j$ fibers of $\ta$.

In this paper we will frequently use tensor projection operators in our analysis. We use $\Imtx: \R^n \to \R^n$ to denote the identity projection. For any tensor $\ta$ of size $n_1 \times \dots \times n_m$, define $\proj^j_\ta: \R^{n_j} \to \L_j (\ta)$ as the orthogonal projection to $\L_j (\ta)$. Similarly we define $\proj^j_{\ta^\perp} = \Imtx - \proj^j_\ta$ to be the projection to the orthogonal complement $\L_j^\perp (\ta)$ of $\L_j (\ta)$.
We define the tensor projection with respect to $\ta$ as follows
\begin{align*}
\pq_\ta^0 &:= \proj_\ta^1 \otimes \dots \otimes \proj_\ta^m , \qquad
\pq_\ta^i := \proj_\ta^1 \otimes \dots \otimes \proj_\ta^{i-1} \otimes \proj_{\ta^\perp}^{i} \otimes \proj_\ta^{i+1} \otimes \dots \otimes \proj_\ta^m 
\end{align*}
for $i \in [m]$, and
\begin{align*}
\pq_\ta := \sum_{j=0}^m \pq_\ta^j, \qquad
\pq_{\ta^\perp} := \Imtx - \pq_\ta \,.
\end{align*}

\subsection{High-order Planted Models}
\label{section:def_model}

We now introduce the definition of \emph{high-order planted models}.
\begin{definition}[High-order Planted Models]
A high-order planted model is denoted as $\mathcal{M}(n,m,r,k,p,q)$, where $n$ is the number of vertices, and $m$ is the order of the model. It is assumed that uniformly at random, $rk$ out of $n$ vertices are grouped into $r$ clusters of equal size $k$, and the remaining $n - rk$ vertices do not belong any cluster. $p, q$ are signal parameters satisfying $0 \leq q < p \leq 1$.

Model $\mathcal{M}(n,m,r,k,p,q)$ generates a random hypergraph $\mathcal{G} = (\mathcal{V},\mathcal{E})$ in the following way. For each $m$-tuple $(v_{i_1},\ldots,v_{i_m}) \subset \mathcal{V}$, if all $m$ vertices are from the same cluster, nature adds the hyperedge $(v_{i_1},\ldots,v_{i_m})$ to $\mathcal{E}$ with probability $p$; otherwise, nature adds the hyperedge with probability $q$.
\label{def:model}
\end{definition}

Our goal is to recover the cluster membership of vertices in model $\mathcal{M}$ from the observed hypergraph $\mathcal{G}$. For any $i\in[r]$, we use $\vy^{(i) \ast} \in \{0,1\}^n$ to denote the true membership vector of cluster $i$, such that  $\vy^{(i) \ast}_j = 1$ if vertex $j$ is in cluster $i$, and $0$ otherwise. 
We introduce the agreement tensor $\tya = \sum_{i=1}^r \vy^{(i)\ast\om}$. It is not hard to see that $\tya$ is $0$-$1$ valued, as the clusters are non-overlapping. Thus, $\tya$ encodes all cluster membership information (up to the permutation of clusters, as there is no way to distinguish between clusters without prior knowledge). 
Let $\ta$ be the adjacency tensor of hypergraph $\mathcal{G}$. From the definition above, $\ta$ is a symmetric tensor. Each entry in $\ta$ is generated to be $1$ with probability $p$ if the corresponding entry in $\tya$ is $1$; otherwise it is generated to be $1$ with probability $q$. The problem now reduces to recover $\tya$ from the observation of $\ta$.

\textbf{Classical models.} Here are some classical models covered by our definition.
\begin{itemize}
    \item \textbf{Disjoint Hypercliques}: $p = 1, 0<q<1$. In this case, $r$ hypercliques of size $k$ are planted in the observed hypergraph $\mathcal{G}$. A hyperclique is the generalization of graph cliques. In a hyperclique, every distinct $m$-tuple is connected by a hyperedge \citep{nie2017symmetric}.
    \item \textbf{Densest Subhypergraph}: $0<q<p<1, r=1$. In this case there exists a dense subhypergraph of size $k$ in the observed hypergraph $\mathcal{G}$.
    \item \textbf{Hypergraph Stochastic Block Model}: $0<q<p<1, n = rk, r\geq 2$. In this case there exists $r$ dense subhypergraphs of size $k$ in the observed hypergraph $\mathcal{G}$.
\end{itemize}

\textbf{Remark on diagonal entries.} One can see that in the adjacency tensor $\ta$, most entries follow Bernoulli distribution with parameter $p$ or $q$. A natural question is: what about the diagonal entries? By ``diagonal'', we refer to the entries with at least one duplicate index, for example, $\ta_{1,1,3,4}$ in a $4$th order model. Mathematically, we are referring to the set $\lbrace \ta_{i_1,\dots, i_m} : \abs{\{   i_1,\dots,i_m    \}} < m\rbrace$. \begin{itemize}
    \item A trivial approach is to force all diagonal entries to be $0$. This is customary in the stochastic block model \citep{abbe2015exact, chen2014statistical}. In the context of graph theory, this means that $\mathcal{G}$ is an \emph{$m$-uniform} hypergraph, where all hyperedges have size $m$.
    \item Another approach is to allow the diagonal entries to be Bernoulli random variables with parameter $p$ or $q$, depending on the corresponding entries in $\tya$. In this case the hypergraph $\mathcal{G}$ is no longer required to be uniform.
\end{itemize}
In the following analysis we adopt the latter approach for its generality. We want to highlight that the technical difference is marginal, however. Either choice will not break the framework of our analysis. 
\section{Efficient Exact Partitioning}

In this section, we propose and analyze an algorithm which efficiently recovers the true underlying cluster structures in high-order planted models. 
Recall that $\tya$ is the true agreement tensor. We say an algorithm achieves \emph{exact partitioning}, if its output $\ty$ is identical to $\tya$.

 \begin{algorithm}
    \caption{Exact Partitioning of High-order Models}\label{euclid}
    \hspace*{\algorithmicindent} \textbf{Input:} adjacency tensor $\ta$\\
    \hspace*{\algorithmicindent} \textbf{Output:} estimated agreement tensor $\ty$  
\begin{align}
\maximize_{\ty} \qquad & \langle \ta, \ty\rangle  \nonumber\\
\st \qquad 
&\tnormnu{\ty} \leq rk^{m/2} \nonumber\\
&\inprod[\onevct^{\otimes m}]{\ty} = rk^m \nonumber\\
&0 \leq \ty \leq 1\,,
\label{opt:primal}
\end{align}
\end{algorithm}

In the following analysis we examine the statistical conditions for problem \eqref{opt:primal} to succeed with high probability. Note that the objective function and constraints in problem \eqref{opt:primal} are convex. It is known that convex optimization problems can be solved efficiently using interior point methods \citep{boyd2004convex}. \textbf{Our analysis establishes the regime in which given the adjacency tensor $\ta$, the true underlying cluster structures $\tya$ can be recovered by problem \eqref{opt:primal} efficiently and perfectly.}

\textbf{Remark on exact partitioning.} It is worth mentioning that our algorithm does not require any rounding step and outputs the exact solution. Hypergraph partitioning algorithms in prior literature either are objective function approximation algorithms \citep{chlamtac2016densest,taylor2016approximations}, or unfold hypergraphs into matrices \citep{kim2017community, ghoshdastidar2017consistency}. Note that the groundtruth $\tya$ is a feasible solution to problem \eqref{opt:primal}. Our analysis states that if certain statistical conditions are satisfied, with high probability no other feasible solution $\ty \neq \tya$ will achieve a better objective value. 

We now present the main theorem, which provides a sufficient condition for problem \eqref{opt:primal} to succeed with high probability.
\begin{theorem}
Consider any hypergraph $\mathcal{G}$ sampled from a high-order model $\mathcal{M}(n,m,r,k,p,q)$. Let $\ta$ be the adjacency tensor of $\mathcal{G}$. If
\begin{equation}
\frac{(p-q)^2}{p(1-q)} = \Omega\left( \frac{nm^5 \log m}{k^{m-1}} \right) \,,
\label{eq:thm}
\end{equation}
then problem \eqref{opt:primal} recovers the underlying cluster structure of $\mathcal{M}$ perfectly with probability at least $1-O(1/n)$.
\label{thm:main}
\end{theorem}

\textbf{Remark on rates.} In high-order planted models, $p$ and $q$ are signal parameters that determine the signal-to-noise ratio (SNR) of the model. This is implied by the left-hand side of \eqref{eq:thm}: as the gap $p-q$ becomes larger, SNR becomes higher and exact partitioning gets easier. On the right-hand side, one can notice that as $n$ gets larger, the whole term becomes smaller (remember that $n = rk$). From an information-theoretical point of view this is intuitive, as a larger number of samples leads to easier recovery of the true signal.

It would also be interesting to compare our rates with those of ordinary planted models.
In the case of $m=2$, our condition becomes
$(p-q)^2 k^2 \geq Cp(1-q) k n$ for some constant $C$, while the condition in \citet{chen2014statistical} is $(p-q)^2 k^2 \geq C(p(1-q)k\log n + q(1-q)n )$. Comparison on the right-hand sides shows that our bound only requires a slightly higher order ($\Omega(kn)$ versus $\Omega(k\log n + n)$).

 
\subsection{Technical Lemmas}
We first present some technical lemmas that will be used to prove our main result in Theorem \ref{thm:main}.
\begin{lemma}
Let $\ta$ be an $m$-th order symmetric random tensor of size $\R^{n \times \dots \times n}$. Assume that each entry of $\ta$ is independent and follows Bernoulli distribution with parameter $p$ or $q$, such that $0\leq q<p \leq 1$. Then with probability at least $1- O(1/n)$, we have 
\begin{equation}
\tnorm{\ta - \Expect{\ta}} 
\leq 
C \sqrt{p(1-q) mn\log m } \,,    
\end{equation}
for some large enough constant $C$.
\label{lemma:concen_A}
\end{lemma}
\begin{proof}
First note that the entries in $\ta - \Expect{\ta}$ are independent (up to symmetry) and zero-mean. Furthermore, the variance of each entry is bounded above by $p(1-q)$. Then \citet[Lemma 1, Theorem 1]{tomioka2014spectral} implies
\begin{align*}
\tnorm{\ta - \Expect{\ta}}
\leq 
\sqrt{8 p(1-q) \left(mn\log(5m) + \log (2/\delta)\right)}
\end{align*}
with probability at least $1-\delta$. Setting $\delta$ to be $O(1/n)$ leads to 
\begin{align*}
\tnorm{\ta - \Expect{\ta}}
\leq 
C \sqrt{ p(1-q) mn\log m }
\end{align*}
with probability at least $1- O(1/n)$, for some large enough constant $C$.
\end{proof}


\begin{lemma}
For any $m$-th order tensor $\tx \in \R^{n_1 \times \dots \times n_m}$, there exists a $\tw_0 \in \R^{n_1 \times \dots \times n_m}$ such that 
\[
\tw_0 = \pq_\tx^0 (\tw_0), \quad \tnorm{\tw_0} = 1, \quad \text{and } \tnormnu{\tx} = \inprod[\tw_0]{\tx} \,.
\]

Furthermore, for any $\tx', \tw_1 \in \R^{n_1 \times \dots \times n_m}$ with $\tnorm{\tw_1} \leq \frac{2}{m(m-1)}$, 
\begin{equation}
\tnormnu{\tx'} \geq \tnormnu{\tx} + \inprod[\tw_0 + \pq_{\tx^\perp} (\tw_1)]{\tx' - \tx} \,.
\end{equation} 
\label{lemma:yuan2017}
\end{lemma}
\begin{proof}
The first claim follows the proof of \citet[Theorem 1]{yuan2017incoherent} by setting $\delta = 1$. It is also proved that any $\tw_1 \in \R^{n_1 \times \dots \times n_m}$ such that $\tnorm{\tw_1} \leq \frac{2}{m(m-1)}$, one has 
$
\tnorm{\tw_0 + \pq_{\tx^\perp} (\tw_1)} \leq 1 \,.
$
Then, since the tensor spectral norm and tensor nuclear norm are dual to each other, and by orthogonality of projection, it follows that 
\begin{align*}
\inprod[\tw_0 + \pq_{\tx^\perp} (\tw_1)]{\tx' - \tx}
&= \inprod[\tw_0 + \pq_{\tx^\perp} (\tw_1)]{\tx'} - \inprod[\tw_0 + \pq_{\tx^\perp} (\tw_1)]{\tx} \\
&\leq \tnorm{\tw_0 + \pq_{\tx^\perp} (\tw_1)} \cdot \tnormnu{\tx'} - \inprod[\tw_0]{\tx} \\
&\leq \tnormnu{\tx'} - \inprod[\tw_0]{\tx} \\
&= \tnormnu{\tx'} - \tnormnu{\tx} \,.
\end{align*}
This completes the proof of the second claim.
\end{proof}


\subsection{Proof of Theorem \ref{thm:main}}
Armed by the technical lemmas above, we now focus on the proof of our main result in Theorem \ref{thm:main}.

\begin{proof}
In the following proofs we use $\tya$ to denote the true agreement tensor as defined in Section \ref{section:def_model}. We also define $\Delta(\ty) = \inprod[\ta]{\dty}$. For any vertex $i\in [n]$, we use $N(i) := \{j : i,j \text{ are in the same cluster}\}$ to denote the neighborhood of $i$. Note that $\abs{N(i)} = k$.

To prove $\tya$ is the optimal solution to problem \eqref{opt:primal}, our goal is to prove $\Delta(\ty) > 0$ for every feasible $\ty \neq \ty^\ast$ satisfying the constraints in problem \eqref{opt:primal}, with high probability. 
It is worth mentioning that the groundtruth $\tya$ is always a feasible solution to problem \eqref{opt:primal}, with $\tnormnu{\tya} = rk^{m/2}$.

To start the analysis we first split $\Delta(\ty)$ as follows
\begin{equation}
\Delta(\ty) = \inprod[\Expect{\ta}]{\dty} + \inprod[\ta - \Expect{\ta}]{\dty}\,.
\label{eq:dy_twoparts}
\end{equation}
We now proceed to characterize these two terms.
Recall that $\ta$ is an $m$-th order symmetric random tensor of size $\R^{n \times \dots \times n}$, and each entry of $\ta$ is independent and follows Bernoulli distribution with parameter $p$ or $q$.

For the first term in $\eqref{eq:dy_twoparts}$, we have
\begin{align}
\inprod[\Expect{\ta}]{\dty}  
&= \inprod[q\onevct^\om + (p-q)\tya]{\dty} \nonumber\\
&= q\inprod[\onevct^\om]{\dty} + (p-q)\inprod[\tya]{\dty} \nonumber\\
&= 0 + (p-q) \sum_{i_1,\dots,i_m} \1[{\tya_{i_1,\dots,i_m} = 1} \land {\ty_{i_1,\dots,i_m} = 0}] \nonumber\\
&= \frac{1}{2}(p-q)\tnormone{\dty} \,,
\label{eq:dy_expectation}
\end{align}
where the last two equalities hold because $\inprod[\onevct^\om]{\ty} = \inprod[\onevct^\om]{\tya} = rk^m$ and $0 \leq \ty \leq 1$. 

To bound second term in $\eqref{eq:dy_twoparts}$, we first rewrite $\inprod[\ta - \Expect{\ta}]{\dty}$ as $\lambda \inprod[\tz]{\dty}$, where $\lambda := \frac{m(m-1)}{2} \cdot C \sqrt{ p(1-q) mn\log m }$. Lemma \ref{lemma:concen_A} guarantees $\tnorm{\tz} \leq \frac{2}{m(m-1)}$ with high probability. 

Before we continue bounding $\lambda \inprod[\tz]{\dty}$, here we take a pause, and instead characterize the subdifferential of $\tnormnu{\tya}$. 
Note that for any $\ty$ and $\tya$, from the definition of tensor spectral norm we have $\tnorm{\ty} = \tnorm{\tya} = k^{m/2}$. 
Setting $\tx = \tya$  in Lemma \ref{lemma:yuan2017}, it can be verified that the conditions
\[
\tw_0 = \pq_{\tya}^0 (\tw_0), \quad \tnorm{\tw_0} = 1, \quad \tnormnu{\tya} = \inprod[\tw_0]{\tya}
\]
holds for the choice $\tw_0 = k^{-m/2}\tya$.
Furthermore, by setting $\tx' = \ty$ and $\tw_1 = \tz$ , it follows that 
\begin{equation}
\tnormnu{\ty}
\geq 
\tnormnu{\tya} + \inprod[k^{-m/2}\tya + \pq_{\ty^{\ast\perp}} (\tz)]{\ty - \tya} \,. 
\label{eq:subdifferential}
\end{equation}
In addition, problem \eqref{opt:primal} requires $\tnormnu{\ty} \leq rk^{m/2}$ for every $\ty$. Since $\tnormnu{\tya} = rk^{m/2}$, we obtain
$\tnormnu{\ty} - \tnormnu{\tya} \leq 0$ for every feasible $\ty$. Combining this with \eqref{eq:subdifferential} we obtain
\begin{equation}
\inprod[k^{-m/2}\tya + \pq_{\ty^{\ast\perp}} (\tz)]{\dty}
\geq 
0 \,.
\label{eq:subgeq0}
\end{equation}

We can now continue to bound 
$\lambda \inprod[\tz]{\dty}$. 
It follows that
\begin{align}
\lambda \inprod[\tz]{\dty}
&= \lambda \inprod[\pq_{\tya} (\tz) + \pq_{\ty^{\ast\perp}} (\tz)]{\dty} \nonumber\\
&\geq \lambda \inprod[\pq_{\tya} (\tz) + \pq_{\ty^{\ast\perp}} (\tz)]{\dty} - \lambda \inprod[k^{-m/2}\tya + \pq_{\ty^{\ast\perp}}(\tz)]{\dty} \nonumber\\
&= \lambda \inprod[\pq_\tya (\tz) - k^{-m/2}\tya]{\dty} \,,
\label{eq:dy_deviation}
\end{align}
where the inequality holds by introducing \eqref{eq:subgeq0}.

Plugging the results from \eqref{eq:dy_expectation} and \eqref{eq:dy_deviation} into \eqref{eq:dy_twoparts}, since the tensor $L_1$ norm and tensor $L_\infty$ norm are dual to each other, we obtain
\begin{align}
\Delta(\ty) 
&\geq \frac{1}{2}(p-q)\tnormone{\dty} + \lambda \inprod[\pq_\tya (\tz) - k^{-m/2}\tya]{\dty} \nonumber\\
&\geq \left( \frac{1}{2}(p-q) - \lambda k^{-m/2}\tnorminf{\tya} - \lambda\tnorminf{\pq_\tya (\tz)} \right) \tnormone{\dty} \nonumber\\
&= \left( \frac{1}{2}(p-q) - \lambda k^{-m/2} - \tnorminf{\pq_\tya (\lambda\tz)} \right) \tnormone{\dty} \,.
\end{align}
Since $\tnormone{\dty}$ is always positive, $\Delta(\ty) \geq 0$ holds if 
\begin{equation}
\frac{1}{2}(p-q) - \lambda k^{-m/2} - \tnorminf{\pq_\tya (\lambda\tz)} 
\geq 0 \,.
\label{eq:suff_condition}    
\end{equation}


We now want to bound the tensor $L_\infty$ norm $\tnorminf{\pq_\tya (\lambda \tz)}$ in \eqref{eq:suff_condition}. Since $\tya$ is a symmetric tensor, the linear spaces spanned by fibers along each mode are identical, i.e., $\L_1(\tya) = \dots = \L_m(\tya) = \L(\tya)$. It follows that $\proj_{\tya}^1 = \dots = \proj_{\tya}^m = \proj_{\tya}$.
By definition of projection operators defined in Section \ref{section:notation}, it follows that 
\begin{align*}
\pq_\tya 
&= \sum_{j=0}^m \pq_\tya^j \\
&= (\Imtx - \proj_\tya) \otimes \proj_\tya \otimes \dots \otimes \proj_\tya 
+ \proj_\tya \otimes (\Imtx - \proj_\tya) \otimes \dots \otimes \proj_\tya + \dots 
 \\
&\quad + \proj_\tya \otimes \dots \otimes (\Imtx - \proj_\tya) + \proj_\tya \otimes \dots \otimes \proj_\tya \\
&= \Imtx \otimes \proj_\tya \otimes \dots \otimes \proj_\tya 
+ \proj_\tya \otimes \Imtx \otimes \dots \otimes \proj_\tya + \dots + \proj_\tya \otimes \proj_\tya \otimes \dots \otimes \Imtx \\
&\quad- (m-1) \proj_\tya \otimes \dots \otimes \proj_\tya \,.
\end{align*}
Note that there are $m$ terms in the second to last line above. Due to symmetry and the fact that $\proj_\tya$ is an orthogonal projection, we have
\[
\tnorminf{\pq_\tya (\lambda \tz)} \leq (2m-1) \cdot \tnorminf{\Imtx \otimes \proj_\tya \otimes \dots \otimes \proj_\tya (\lambda \tz)} \,.
\]
We denote $\bar{\ta} := \Imtx \otimes \proj_\tya \otimes \dots \otimes \proj_\tya (\lambda \tz)$ to be the projected tensor.
We now consider the linear space $\L(\tya)$. By definition of $\tya$, the spanned space has $r$ basis vectors $\vy^{(1)\ast}, \dots, \vy^{(r)\ast}$. Thus the orthogonal projection can be characterized as $\proj_\tya = \frac{1}{k} \sum_{i=1}^r \vy^{(i)\ast} \vy^{(i)\ast\top} $. 
To bound $\tnorminf{\Imtx \otimes \proj_\tya \otimes \dots \otimes \proj_\tya (\lambda \tz)} = \tnorminf{\bar{\ta}}$, note that every single entry in $\bar{\ta}$ is the average of $k^{m-1}$ independent random variables. Mathematically, we have 
\begin{equation}
    \bar{\ta}_{i, i_2, \dots, i_m} = \frac{1}{k^{m-1}}\sum_{j_2 \dots j_m \in N(i)} \lambda \tz_{i,j_2, \dots, j_m}
\label{eq:A_avg}
\end{equation}
for every $i,i_2,\dots,i_m$, where $N(i)$ is the neighborhood of $i$. Bernstein's inequality implies
\begin{equation}
\Prob{\sum_{j_2 \dots j_m \in N(i)} \lambda \tz_{i,j_2, \dots, j_m} \geq \sqrt{2(m+1)k^{m-1} p(1-q) \log n} + \frac{2}{3} (m+1)\log n}
\leq 
n^{-(m+1)} \,.
\label{eq:bern_Z}
\end{equation}
Putting \eqref{eq:A_avg} and \eqref{eq:bern_Z} together, one can see that 
$\bar{\ta}_{i, i_2, \dots, i_m} \geq 
\sqrt{\frac{2(m+1) p(1-q) \log n}{k^{m-1}}} 
+ \frac{2(m+1)\log n}{3k^{m-1}}$ 
holds with probability less than $n^{-(m+1)}$. 
Taking a union bound for $i_2,\dots, i_m$, we have 
\[
\tnorminf{\pq_\tya (\lambda \tz)} \leq (2m-1) \cdot \tnorminf{\bar{\ta}}
\leq 3m\left(\sqrt{\frac{ p(1-q) \log n}{k^{m-1}}} + \frac{\log n}{k^{m-1}} \right)
\]
with probability at least $1-n^{-1}$.
Substitute the results above into \eqref{eq:suff_condition}, we obtain the sufficient condition
\begin{align}
p-q \geq C m(m-1) \sqrt{\frac{ p(1-q) mn \log m}{k^m}} + 6m \sqrt{\frac{ p(1-q) \log n}{k^{m-1}}} + 6m\frac{\log n}{k^{m-1}} \,.
\end{align}
Note that the order of the first two terms on the right-hand side are both dominated by $m^2 \sqrt{\frac{ p(1-q) mn \log m}{k^{m-1}}}$. Combining the first two terms leads to the sufficient condition of 
$
p-q \geq C m^2 \sqrt{\frac{ p(1-q) mn \log m}{k^{m-1}}} + 6m \frac{\log n}{k^{m-1}} \,.
$
Then the first term on the right-hand side dominates the second term, if the mild condition 
$(m^3 \log m)  p(1-q) k^{m-1} \geq 1$
is satisfied.
Then the remaining condition is 
\begin{align}
\frac{p-q}{C\sqrt{ p(1-q) m^5 \log m}} 
\geq 
\sqrt{\frac{n}{k^{m-1}}} \,.
\end{align}
This completes our proof.
\end{proof}

\section{Concluding Remarks}
In this paper, we studied the problem of exact partitioning of high-order planted models. We proposed a novel convex tensor optimization problem which, given the sufficient condition in Theorem \ref{thm:main} is satisfied, recovers the underlying cluster structure perfectly with probability tending to $1$. 

A key motivation behind our work is to develop efficient exact partitioning algorithms for high-order planted models with provable guarantees. To the best of our knowledge the approach is novel for hypergraph models. There have also been some works taking some more conventional approaches, for example, information-theoretic bounds in densest subhypergraphs \citep{buhmannrecovery} and hypergraph SBMs \citep{corinzia2019exact}, as well as detection (i.e., determining whether there exists certain planted structures) in hypergraph SBMs \citep{angelini2015spectral} and hypergraph planted cliques \citep{zhang2018tensor}. These could be a direction in our future works.


\bibliography{0_main.bib}

\bibliographystyle{abbrvnat}


\end{document}